\documentclass{article}
\pdfoutput=1
\usepackage{microtype}
\usepackage{graphicx}
\graphicspath{{experiments-release/customer_reviews_exp/plots/}}
\usepackage{subcaption}
\usepackage{booktabs} 

\usepackage{hyperref}
\usepackage{bbm}

\usepackage[preprint]{icml2026}
\date{}
\usepackage{amsmath}
\usepackage{amssymb}
\usepackage{mathtools}
\usepackage{amsthm}
\usepackage{extarrows}

\usepackage[capitalize,noabbrev]{cleveref}

\let\oldnewtheorem\newtheorem
\RenewDocumentCommand{\newtheorem}{m o m
  o r<>}{\IfValueTF{#2}{\oldnewtheorem{#1}[#2]{#3} \AddToHook{env/#1/begin}{\crefalias{#2}{#1}} }{\IfValueTF{#4}{\oldnewtheorem{#1}{#3}[#4] }{\oldnewtheorem{#1}{#3}}} \crefname{#1}{#3}{#5} }

\theoremstyle{plain}
\newtheorem{theorem}{Theorem}[section]<Theorems>
\newtheorem{lemma}[theorem]{Lemma}<Lemmas>
\newtheorem{observation}[theorem]{Observation}<Observations>
<Claims>
<Conditions>
<Examples>
<Facts>
\newtheorem{corollary}[theorem]{Corollary}<Corollaries>
\newtheorem{proposition}[theorem]{Proposition}<Propositions>
\theoremstyle{definition}
\newtheorem{definition}[theorem]{Definition}<Definitions>
\newtheorem{remark}[theorem]{Remark}<Remarks>
<Problems>
<Questions>
<Exercises>
\newtheorem{assumption}[theorem]{Assumption}<Assumptions>

\crefname{equation}{eq.}{eqs.}

\usepackage[textsize=tiny]{todonotes}

\newcommand{\detailcite}[2]{#1 in \citet{#2}}

\RequirePackage{etoolbox}

\DeclareMathOperator{\set}{set}
\DeclareMathOperator{\supp}{supp}

\newcommand{\world}{document distribution}
\newcommand{\worlds}{\world{}s}

\newcommand{\dworld}{\ensuremath{D_{\mathrm{world}}}}

\newcommand{\Mis}{\textup{Mis}}

\DeclarePairedDelimiter{\inp}{(}{)}
\DeclarePairedDelimiter{\insq}{[}{]}
\DeclarePairedDelimiter{\abs}{\vert}{\vert}

\newcommand{\mparagraph}[1]{

\noindent\textbf{#1}\ \ }

 \icmltitlerunning{Innovation: An Almost Characterization of Hallucination}

\begin{document}

\twocolumn[
	\icmltitle{Innovation: An Almost Characterization of Hallucination}

\icmlsetsymbol{equal}{*}

	\begin{icmlauthorlist}
		\icmlauthor{Nishant P. Das}{comp}
		\icmlauthor{Piyush Srivastava}{comp}
	\end{icmlauthorlist}

	\icmlaffiliation{comp}{School of Technology and Computer Science, Tata
          Institute of Fundamental Research, Mumbai, Maharashtra - 400 005, India}

	\icmlcorrespondingauthor{Nishant Das}{nishant.das@tifr.res.in}
	\icmlcorrespondingauthor{Piyush Srivastava}{piyush.srivastava@tifr.res.in}

\icmlkeywords{Hallucinations, Probabilistic modeling}

	\vskip 0.3in
]

\printAffiliationsAndNotice{}  

\begin{abstract}

Hallucination is a central limitation of large language models (LLMs), and
substantial effort has been devoted to understanding and mitigating it. Towards
this, Kalai and Vempala (STOC 2024) introduced a probabilistic framework
formalizing calibration and hallucination, and showed that, with high
probability, calibrated LLMs hallucinate roughly at the rate of the ``missing
mass'', a measure of how incomplete the training data is relative to its
source. This raises two fundamental questions: (i) what property of a calibrated
LLM makes hallucinations unavoidable? and (ii) can hallucinations be avoided by
giving up calibration? We answer these questions by introducing a simpler
property we call \emph{innovation} that measures the tendency of a model to
produce outputs outside the training data. We show that innovation is implied by
the condition for hallucination identified by Kalai and Vempala, and, further,
that it is an almost characterization of hallucination: hallucination implies
innovation, and conversely, innovation implies hallucination with high
probability.  We also provide lower bounds on the hallucination rate based on
the ``innovation rate'', and by relating innovation rate back to missing mass,
we obtain new hallucination rate lower bounds based on missing mass that extend
the results of Kalai and Vempala.

 \end{abstract}

\section{Introduction}

One of the principal limitations of Large Language Models (LLMs) is the
phenomenon of \emph{hallucination}~\cite{huang2025survey}: the model produces
statements that sound plausible but that are factually incorrect or logically
inconsistent. This behavior persists even in state-of-the-art systems and thus
remains a central challenge in the theory and practice of machine learning.
Substantial work has therefore been devoted to mitigating the effect of
hallucinations in practice, and we discuss some of this later in the paper.

Another line of work, however, has sought to understand hallucination abstractly
as a statistical phenomenon.  The approach in this line of work is to abstract
away the details of the model architecture, and to view LLMs as elements of a
general class of (possibly randomized) procedures.  To be useful, such an
abstract framework must be rich enough to model hallucination and preferably
other statistical properties of LLMs as well.  One can then hope to establish
general relationships between such statistical properties.  The advantage of
such abstract frameworks is that if the framework is also realistic enough, then
the relationships obtained can be expected to hold quite independently of the
underlying architecture or learning procedures, and could potentially apply also
to architectures and learning procedures that may be discovered in the future.

An important probabilistic framework for studying hallucination is due to
\citet{kalaivempala2024caliberatedllms}.  In their framework, the ``world'' is
modeled as a \emph{meta-distribution} \dworld{} over \worlds{}.  The ``true''
\world{} \(p^{*}\) is drawn from this meta distribution \dworld{}, and the
support of this ``true'' \world{} \(p^{*}\) represents the ``facts''.  A
\emph{corpus} is generated by sampling repeatedly from the ``true'' \world{}. A
language model (more precisely, the training process of a language model) is
abstracted as an algorithm that receives such a corpus and outputs a
distribution over statements.  A model is said to be \emph{calibrated} if its
output distribution approximately matches the ``true'' \world{} under a notion
of distance formalized by Kalai and Vempala (see \cref{def:miscalibration}
below). A key quantity in their framework is the
\emph{missing mass}, which is the probability mass, under the ``true'' \world{},
of the set of statements outside the corpus. Under some natural regularity
assumptions on the meta-distribution, the main result of Kalai and Vempala is
that ``\emph{calibrated language models must hallucinate},'' roughly at the rate
of the missing mass. This \emph{rate} result also implies an \emph{existence}
result: it shows that if a calibrated model faces sufficiently large missing
mass, hallucination is inevitable.

\mparagraph{Contributions}
The above results lead to a few natural questions: what is the simplest natural
property of an LLM that makes hallucination unavoidable?  For example, can a model
avoid hallucinating by not being calibrated?

We first give qualitative answers to these questions through an intermediate
property that we call \emph{innovation} (see \cref{def-innov} for a formal
definition).  In comparison to the sophisticated formalization of calibration
developed by \citet{kalaivempala2024caliberatedllms} (see
\cref{def:miscalibration} below), this is a very simple quantity: it just
measures the probability assigned by the model to statements not observed in the
training data.  Our main qualitative result is that this simple quantity,
however, lies at the heart of the hallucination phenomenon in the Kalai-Vempala
framework.  At an informal high level, under the same regularity assumption as
in the work of Kalai and Vempala, we refine their existence result
\begin{align*}
    &\text{Calibration + significant missing mass} \\
    \xLongrightarrow{\text{w.h.p.}}\;
    & \text{Hallucination}
\end{align*}
into a finer chain of implications:
\begin{equation}
  \label{eq-refined-chain}
  \begin{aligned}
    & \text{Calibration + significant missing mass} \\
\;\Longrightarrow\;                & \text{Innovation}
                           \;\xLongrightarrow{\text{w.h.p}}\;  \text{Hallucination}\\
                           \;\Longrightarrow\;  &               \text{Innovation}.
  \end{aligned}
\end{equation}
We re-emphasize two important points made in the above diagram.  First, our
simpler condition of innovation holds whenever the condition of calibration
combined with significant missing mass used by
\citet{kalaivempala2024caliberatedllms} holds: it is therefore weaker than
their condition.  Second, our condition essentially characterizes hallucination:
our results show that in the Kalai-Vempala framework with their regularity
assumptions, \emph{innovation and hallucination are two sides of the same coin}.
\cref{sec:existence} presents our main qualitative results, showing how
innovation is an almost characterization of hallucination.

We then turn, in \cref{sec:rates}, to a more quantitative exploration of the
notion of innovation: our contribution here concerns the \emph{rate} of
hallucination once innovation occurs. While Kalai and Vempala express the
hallucination rate in terms of the missing mass which is a property of the ``true'' \world{} and the training data, we derive two lower bounds: a \emph{Markov-style} bound
and a \emph{high-confidence} bound, in terms of the model's \emph{innovation
  rate} which is a property only of the model and the training data. Under their
baseline assumptions, Kalai and Vempala obtain hallucination rate lower bounds that
depend explicitly on the corpus size \(n\) and degrade as \(n\) grows. This
leaves open the possibility that hallucination might be eliminated given a
sufficiently large corpus. They remove this dependence only by imposing an
additional regularity assumption (\cref{assum-regular-prob} described below). In
contrast, \emph{we remove the dependence on the corpus size \(n\) without this additional
assumption}. Our results show that the innovation rate fundamentally governs the
hallucination rate and that the possibility suggested above is untrue:
increasing the amount of training data alone cannot eliminate hallucination
once innovation occurs.

Finally, we show how to relate innovation rate back to missing mass. Combining
this relation with our innovation-based bounds yields new missing mass lower
bounds on hallucination. These results recover the qualitative conclusions of
Kalai and Vempala under weaker assumptions, and continue to hold in regimes
where their original bounds become vacuous.  

After a brief discussion of related work, we then develop in
\cref{sec:preliminaries} the elements of the Kalai-Vempala framework relevant to
our work. 

 \mparagraph{Related Work}
Hallucination has been extensively studied empirically: benchmarks such as
TruthfulQA attempt to quantify factual errors in LLM
outputs~\cite{lin2021truthfulqa}, and practical approaches for mitigation have
explored methods such as retrieval augmentation, grounding, and uncertainty
estimation~\cite{lewis2020rag,nakano2021webgpt,farquhar2024detect}.  We refer to
\citet{huang2025survey} for a survey of causes and mitigation strategies.

On the theoretical side, \citet{kleinberg2024generation} asked the following
question: given only positive examples from an unknown language, and no feedback
on errors, is it possible to \emph{eventually} generate \emph{new} strings that
all belong to the target language? While classical results in the theory of
language identification~\cite{GOLD1967447} show that \emph{identifying} the
unknown language is impossible under positive data alone, Kleinberg and
Mullainathan demonstrate that \emph{generation} is nonetheless possible.  This
framework has been extended by various subsequent works, e.g., by
\citet{charikarExploringFacetsLanguage2025} (who also consider modifications of
the model in which hallucination becomes inevitable),
\citet{ramanGenerationLensLearning2025},
\citet{raman2025generationnoisyexamples},
\citet{kleinberg2025densitymeasureslanguagegeneration}, and
\citet{languagegenerationtradeoffs}.  Computability theoretic formalizations of
hallucination have also been
analyzed~\cite{xu2024inevitable,suzuki2025negligible}.

Our work builds upon the more statistical formalization of hallucination by
\citet{kalaivempala2024caliberatedllms}, who showed that, in their framework,
calibration implies hallucination.  \citet{wuNoFreeLunch2025} develop a learning
theoretic framework in which they demonstrate the inevitability of hallucination
without requiring calibration as a precondition.
\citet{kalai2025languagemodelshallucinate} further generalize the original
Kalai-Vempala framework to take into account phenomena such as pretraining.  In
contrast to these works, the question we ask in this paper is: what is the
simplest condition that already implies hallucination in the original
Kalai-Vempala framework?  As discussed above, and justified formally in the rest
of this paper, our results show that the simple property of \emph{innovation}
provides an answer to this question.

 \section{The Kalai-Vempala Framework}
\label{sec:preliminaries}

In this section, we introduce the Kalai-Vempala framework. Let \(\Omega\) be a finite
set of \emph{statements}. In practice, \(\Omega\) could represent all possible finite
sequences of words from the vocabulary of the LLM of length up to the context
length of the
LLM. The Kalai-Vempala framework however works in the \emph{unprompted generation}
setting where the prompt is fixed to be the empty string (equivalently, one may
view the prompt as fixed throughout). In this setting, the training data
consists solely of responses, and during inference the model generates a
response without an explicit prompt.

\begin{definition}[Language Model]
	A \emph{language model} is a mapping
	\(
	\mathcal{A} : \mathcal{M}(\Omega) \to \Delta(\Omega),
	\)
	where \(\mathcal{M}(\Omega)\) is the set of finite multisets over \(\Omega\). Thus, given the
	training data \(X \in \mathcal{M}(\Omega)\), the model outputs a predictive distribution \(g := \mathcal{A}(X) \in \Delta(\Omega)\).
\end{definition}

The Kalai-Vempala framework captures LLMs in the following manner: \(\Omega\) is
the set of strings of length up to \(c\) (the \emph{context length} of the LLM)
over a finite set \(\mathcal{V}\) (the \emph{vocabulary} of the LLM).  As mentioned above,
the Kalai-Vempala framework works in the unprompted generation setting.  The autoregressive generating procedure of LLMs then induces a well-defined
probability distribution over \(\Omega\). Hence, LLMs fit into the above abstract
notion of Language Models.

We now motivate and formalize the notions of \emph{\worlds} and \emph{corpuses}
in the framework of Kalai and Vempala.  Large language models are trained on
large-scale datasets collected from a variety of sources such as Wikipedia,
Reddit, news articles, books, blogs, and scientific papers. Different sources
contain different statements, and the same statement may appear with different
frequencies across sources; e.g., technical statements may be more common in
scientific articles than in Reddit, while personal statements may appear more
frequently in the latter.

Each such source (or any fixed combination of sources) defines a
\emph{\world}: this is a probability distribution over the set \(\Omega\) of all statements, and represents which statements appear in the source and how
often they occur.
The support of a \world{} defines
the set of \emph{facts} in the \world{}, while the complement of the support is
the set of \emph{hallucinations}.   This means that in the Kalai-Vempala framework, there is no \emph{semantic}
notion of truth: a fact is defined simply by belonging to the support of the
\world{}.

A \emph{corpus} is generated by repeatedly
sampling from the \world{}.
Given a corpus from a particular \world{}, a language model trained on it is said
to \emph{hallucinate} (according to that \world{}) if it places positive
probability mass on the set of \emph{hallucinations} of that \world{}.  For
instance, for a model trained on a corpus sampled from Wikipedia, all statements that occur in Wikipedia are
taken to be factual within the Wikipedia \world{}. Thus, in the Kalai-Vempala
framework, a model trained on a corpus drawn from Wikipedia is said to
hallucinate if and only if it produces statements that are absent in Wikipedia.

Note that while the corpus itself is observed, the underlying \world{} that
produced it is not. Indeed, many distinct \worlds{} may give rise to the same
observed corpus. As a result, from the corpus alone, one cannot comment about
how much the model hallucinates. To reason about hallucination in this setting,
Kalai and Vempala introduce a \emph{meta-distribution} \(\dworld\)
over \worlds{}, capturing uncertainty about which \world{} generated the
observed corpus. Once a distribution over \worlds{} is fixed, we can reason
probabilistically via the posterior over \worlds{} given the corpus and make
probabilistic statements about hallucination across \worlds{}.

\begin{remark}
  The unprompted generation setting and the assumption that all observed
  statements are true are both significant simplifications. Nevertheless, both
  our results and those of \citet{kalaivempala2024caliberatedllms} already hold
  under these assumptions. In more realistic settings, the conclusions would at
  the very least remain valid, and may in fact become stronger. For one
  approach to extending this framework to the prompted setting, see
  \citet{kalai2025languagemodelshallucinate}. We also note that LLMs exhibit
  additional forms of hallucination beyond the statistical notion studied here, 
  incorporating which would only increase the overall hallucination rate. Our
  goal, however, is to understand the simplest possible setting under which the
  probabilistic properties of a model, independent of its internal mechanics,
  imply hallucination.
\end{remark}

\mparagraph{Notation} We now review the formal notation, largely following
\citet{kalaivempala2024caliberatedllms}, for the above notions. As before,
\(\Omega\) denotes a finite (but large) set of statements. For any set \(S\), we
denote by \(\Delta(S)\) the set of probability distributions over \(S\). A
\emph{\world} is a probability distribution over \(\Omega\) (i.e. an element of
\(\Delta(\Omega)\)). We denote by \(\dworld \in \Delta(\Delta(\Omega))\) the distribution over the
\worlds{}. Each \world{} \(p\) in \(\supp(\dworld{})\) determines its set of
\emph{facts} \(F = F(p):=\supp(p)\) and its set of \emph{hallucinations}
\(H = H(p) :=\Omega \setminus F\). (The symbols \(F\) and \(H\) are chosen for consistency
with \cite{kalaivempala2024caliberatedllms}.)\footnote{As noted by
  \cite{kalaivempala2024caliberatedllms}, with some extra notation, one can
  further distinguish between documents and the facts that are contained in
  those documents. However, since the mathematical development of their
  framework depends only on notions of facts and probability weights assigned to
  them, we follow their lead and, without loss of generality, do not
  emphasize this distinction.} A corpus \(X\) is generated by first drawing \(p^*\) from \(\dworld{}\) and then
\(X = (x_1,\dots,x_n) \sim (p^*)^{\times n}\), where \(n\) is the size of the
corpus. Recall that \(p^*\) is unknown (to the model trainer).
\(O := \operatorname{set}(X)\) is the set of \emph{observed statements}
(differing from \(X\) only in de-duplication of statements that appear multiple
times in \(X\)) and \(U := \Omega \setminus O\) are the \emph{unobserved
	statements}. \(p(U)\) thus denotes the \emph{missing mass} for a \world{}
\(p\) and corpus \(X\).  A Language Model \(\mathcal{A}\) trained on a corpus
\(X\) outputs a predictive distribution
\(g = \mathcal{A}(X) \in \Delta(\Omega)\) (again, the symbol \(g\) is chosen for consistency with
\citet{kalaivempala2024caliberatedllms}). Hence, \(g\) represents a trained
language model.  Our central quantity of interest is \(g(H)\) which denotes the
\emph{rate of hallucination}.

\subsection{Regularity Assumptions}
We now present the regularity assumptions imposed on the meta-distribution
\(\dworld{}\) in the Kalai-Vempala framework.  The first of these captures the
idea that semantic truth is rare relative to the space of all well-formed
statements.  \citet{kalaivempala2024caliberatedllms} formalize this by assuming
that each \world{} \(p \sim \dworld{}\) contains only finitely many facts, and that
number is small relative to the full statement space \(\Omega\).
\begin{assumption}[\(K\)-sparsity: \detailcite{Assumption 1}{kalaivempala2024caliberatedllms}]
	\label{assum-k-sparse}
	There exists \(K\) such that every \world{} \(p \in \supp(\dworld{})\)
	satisfies \(|F(p)| \le K\), with \(K/\abs{\Omega} \ll 1\).
\end{assumption}
We note that \citet{kalaivempala2024caliberatedllms} use a slightly different
but equivalent parameterization: they say that the world is \(s\)-sparse if
\(\abs{F(p)} \leq \exp\inp{-s}{\abs{H(p)}}\) for every \(p\) in its support.
The two parameterizations are equivalent via
\(K = \abs{\Omega}/(1 + \exp\inp{s})\).  Further, this assumption also implies that
\(K/|U| \le e^{-s}\) for every corpus generated from a \world{} sampled from
\dworld{}.

The second regularity assumption in the Kalai-Vempala framework formalizes the
idea that outside the observed corpus the model has no reliable signal to
distinguish facts (``true'' statements) from hallucinations (``false''
statements).  This goes back to the idea that in this framework, there is no
\emph{semantic} notion of truth.

\begin{assumption}[{\detailcite{Definition 3: Regular Facts}{kalaivempala2024caliberatedllms}}]
	\label{assum-regular-facts}
	For every corpus \(X\) and posterior \(\nu := \dworld{}(\cdot \mid X)\), all
	unobserved statements \(y,y' \in U\) satisfy
	\[
		\Pr_{p \sim \dworld{}}[y \in F \;\mid X] = \Pr_{p \sim \dworld{}}[y' \in F \;\mid X],
	\]
	where \(F\) denotes \(\supp(p)\).  Thus, conditioned on the observed
	data, all unseen statements are equally likely to be factual in the true
	\world{}.
\end{assumption}

While our results do not rely on the following assumption, for completeness and to facilitate discussions, we describe an additional assumption under which Kalai and Vempala were able to strengthen their results.

\begin{assumption}[{\detailcite{Definition 4: Regular Probabilities}{kalaivempala2024caliberatedllms}}]
	\label{assum-regular-prob}
	For every corpus \(X\) and posterior \(\nu := \dworld{}(\cdot \mid X)\), all
	unobserved statements \(y,y' \in U\) satisfy
	\[
		\mathbb{E}_{p \sim \dworld{}}[p(y) \,\mid\, X] = \mathbb{E}_{p \sim \dworld{}}[p(y') \,|\, X].
	\]
\end{assumption}

\mparagraph{(Mis)calibration} We now describe the final ingredient of the
Kalai-Vempala framework: a formalization of the notion of calibration.  A
\emph{partition} of \(\Omega\) is a collection \(\Pi = \{B_1,\dots,B_m\}\) of disjoint
nonempty subsets of \(\Omega\) whose union is \(\Omega\). Each element of a partition is
called a \emph{cell}. We write \(\mathcal{P}(\Omega)\) for the set of all such partitions.
\begin{definition}[{\detailcite{Definition 1: Calibration and Coarsening}{kalaivempala2024caliberatedllms}}] \label{def:coarsening}
	Let \(p \in \Delta(\Omega)\) and \(\Pi \in \mathcal{P}(\Omega)\). Let \(B_y\) denote the unique cell containing the statement \(y \in \Omega\).
	The \emph{\(\Pi\)-coarsening} of \(p\) is the distribution \(p^\Pi \in \Delta(\Omega)\)
	defined by
	\(
	p^\Pi(y) := \frac{p(B_y)}{|B_y|}.
	\)
	Thus \(p^\Pi\) averages \(p\) within each cell of \(\Pi\) and redistributes
	that mass uniformly inside the cell.  Given \(p,g \in \Delta(\Omega)\), we say that
	\(g\) is \emph{calibrated} to \(p\) if there exists
	\(\Pi \in \mathcal{P}(\Omega)\) such that
	\(
	g = p^\Pi.
	\)
\end{definition}
Notice that \(g\) is calibrated to \(p\) \emph{if and only if} \(g\) is equal to
the coarsening of \(p\) induced by the partition induced by the level sets of
\(g\). This leads to a natural definition for Miscalibration.

\begin{definition}[{Miscalibration: \detailcite{eq.~(2)}{kalaivempala2024caliberatedllms}}]
	\label{def:miscalibration}
	Given \(g \in \Delta(\Omega)\), let \(\mathcal{B}_g\) be the partition of
	\(\Omega\) into level sets of \(g\):
	\( \mathcal{B}_g := \{ B \subseteq \Omega : \exists \alpha \text{ with } B = \{y \in \Omega : g(y) = \alpha\},\; B \neq \emptyset \}.
	\) The \emph{miscalibration} of \(g\) relative to \(p\) is defined as
	\[
		\Mis(g,p)
		:=
		\bigl\| g - p^{\mathcal{B}_g} \bigr\|_{\mathrm{TV}} = \max_{S \subseteq \Omega}|g(S) - p^{\mathcal{B}_g}(S)|.
	\]
\end{definition}
Hence, \(\Mis(g,p) = 0\) if and only if \(g\) is calibrated to \(p\).

\mparagraph{Calibration Implies Hallucination} We are now ready to state the main
results of Kalai and Vempala.  The main result in their paper is their Theorem
1; however, this is a technical result that is not directly interpretable.
Kalai and Vempala therefore present more interpretable consequences of this main
theorem as corollaries, which we state below.

The first of these corollaries (Corollary 2 in their paper) can be directly
compared to our results to be stated later, since it operates under the same
assumptions on the meta distribution as our results.\footnote{Note that the
	statements of Corollaries 1 and 2 of \citet{kalaivempala2024caliberatedllms}
	have an additional error term since those bounds are stated in terms of an
	estimator for \(p(U)\) rather than in terms of \(p(U)\) directly.  However,
	the bounds without that error term presented here are implied by their
	proof.}

\begin{corollary}[{\detailcite{Corollary 2}{kalaivempala2024caliberatedllms}}]
	\label{cor:kv-cor2}
	Let \(\dworld{}\) be a \(K\)-Sparse (Assumption \ref{assum-k-sparse})
	world satisfying Regular Facts (\cref{assum-regular-facts}).
Let \(p^* \sim \dworld{}\) be the true \world{},
	\(X \sim (p^*)^{\times n}\) the training data, and
	\(g = \mathcal{A}(X)\) be the predictive distribution of a language model
	\(\mathcal{A}\). Then for any \(\delta \in (0,1]\), with probability at least
	\(1 - \delta\) conditioned on the input training corpus \(X\), we
	have \[
		g(H)
		\;\ge\;
		p(U)
		\;-\;
		\Mis\bigl(g,p\bigr)
		\;-\;
		\frac{K(n+1)}{\delta\,|U|}.
	\]
\end{corollary}
Note the dependence on \(n\), the size of the corpus: the bound becomes vacuous
once \(n\) becomes greater than \(|U|/K -1\).  This might lead one to hope that
hallucinations can perhaps be avoided by collecting a large corpus. By assuming
the Regular Probabilities assumption (\cref{assum-regular-prob} cited above),
Kalai and Vempala arrive at the following corollary which removes the dependence
on \(n\).
\begin{corollary}[{\detailcite{Corollary
					1}{kalaivempala2024caliberatedllms}}] \label{cor:kv-cor1} Let
	\(\dworld{}\) be a \(K\)-Sparse (Assumption \ref{assum-k-sparse}) distribution
	with Regular Facts (Assumption ~\ref{assum-regular-facts}) and Regular
	Probabilities (Assumption ~\ref{assum-regular-prob}).  Let
	\(p^* \sim \dworld{}\) be the true \world, \(X \sim (p^*)^{\times n}\) the training data,
	and \(g = \mathcal{A}(X)\) be the predictive distribution of a language model
	\(\mathcal{A}\). Then for any \(\delta \in (0,1]\), with probability at least
	\(1 - \delta\) conditioned on the input training corpus \(X\), we
	have \[
		g(H)
		\;\ge\;
		p(U)
		\;-\;
		\Mis\bigl(g,p\bigr)
		\;-\;
		\frac{2K}{\delta\,|U|}.
	\]
\end{corollary}
\mparagraph{A Remark on Conditioning} Note that the probability bounds in the
above two results are with respect to the conditional distribution
\(\dworld{}(\cdot \;\mid X)\) of the \world{} \(p\), conditioned on the observed
training corpus \(X\).  Our results
(e.g. \cref{thm:innovation-implies-hallucination,thm:markov-bound,thm:sharp-bound}
and their corollaries) also derive probability bounds under the same conditional
distribution.

\mparagraph{Relaxing the Regular Facts Assumption}
\citet{kalaivempala2024caliberatedllms} also consider a weaker version of the
Regular Facts assumption (\cref{assum-regular-facts}), called the \(r\)-Regular
Facts condition, under which the posterior probability that an unseen statement
is factual is allowed to deviate from the posterior mean value
\(\mathbb{E}_\nu[|F \cap U|]/|U|\) by a factor of at most \(r\).  Both their results
as well as ours can be extended to use this weakened setting (of course, at a
cost depending upon the parameter \(r\)): we discuss this in
Appendix~\ref{app:r-regular-facts}.

The above two results can thus be summarized informally as follows. In a
\(K\)-sparse meta-distribution with Regular Facts, any language model that is
calibrated (\(\Mis(g,p)=0\)) must hallucinate with the hallucination \emph{rate}
roughly equal to the \emph{missing mass} \(p(U)\), up to an error term of order
\(Kn/|U|\). If we additionally assume that the meta-distribution has Regular
Probabilities, the error term collapses to \(O(K/|U|)\).

\mparagraph{Implications of the Kalai-Vempala Bounds}
These results (Corollary \ref{cor:kv-cor1} and Corollary \ref{cor:kv-cor2})
yield two distinct implications about the behavior of calibrated language
models.
\begin{itemize}
	\item \emph{Existence:} If the model is calibrated and the missing mass \(p(U)\) exceeds the error term (of order \(K/|U|\), or \(Kn/|U|\) in the absence of Regular Probabilities), then hallucination is unavoidable.
	\item \emph{Rate:} Under calibration, the hallucination rate \(g(H)\)
	      is approximately the missing mass \(p(U)\).
\end{itemize}
This distinction between existence and rate will be important in what follows. In \cref{sec:existence}, we strengthen upon the existence result and \emph{almost characterize} the existence of hallucination. In \cref{sec:rates}, we derive new lower bounds on the hallucination rate which complement the Kalai-Vempala bounds.

 \section{Innovation and Hallucination}
\label{sec:existence}
We now describe the main conceptual contribution of this paper: the notion of
\emph{innovation}.  Informally, a model innovates if it assigns nonzero
probability to events or outcomes that it has never directly observed, i.e., it
``extends'' beyond its training support.

\begin{definition}[\textbf{Innovation}]\label{def-innov}
	Let \(\mathcal{A}\) be a language model and \(g := \mathcal{A}(X)\) its
	predictive distribution over \(\Omega\) given input \(X \in \mathcal{M}(\Omega)\). Let
	\(U := \Omega \setminus \mathrm{set}(X)\) denote the set of unseen outcomes. We
	say that \(\mathcal{A}\) \emph{innovates} on \(X\) if
	\[
		g(U) = \mathcal{A}(X)(U) > 0.
	\]
	That is, the model assigns nonzero probability mass to outcomes not present
	in the corpus.
\end{definition}

In this section, we focus on establishing the following qualitative results:
\begin{enumerate}
\item We prove a two-way connection between innovation and hallucination. It is
  easy to see that every hallucinating model necessarily innovates; we show in
  the other direction that every model that innovates must hallucinate with high
  probability. Thus, innovation provides an almost characterization of the
  \emph{existence} of hallucination.

\item Kalai and Vempala showed that calibration together with sufficiently large
  missing mass forces hallucination. We show that calibration together with
  \emph{any} positive missing mass implies innovation. This shows that the
  existence-level implication of Kalai-Vempala is a special case of our
  innovation-based characterization.
\end{enumerate}

\subsection{Proof of Characterization}

We begin with the easy direction: hallucination is only possible if the model innovates.

\begin{observation}[Hallucination implies innovation]
	\label{prop:hall-to-inn}
	Let \(\mathcal{A}\) be a Language Model and \(g := \mathcal{A}(X)\) its predictive distribution over \(\Omega\) given input \(X \in \mathcal{M}(\Omega)\). If \(g(H) > 0\) for some world \(p\), then \(g(U) > 0\).
\end{observation}
The converse direction is subtler, because \(H\), the set of hallucinations,
itself is unknown. We show however that, under the \(K\)-Sparsity and Regular
Facts assumptions, if the model innovates, then with high posterior probability
it must hallucinate.
\begin{theorem}[Innovation implies hallucination with high probability]
	\label{thm:innovation-implies-hallucination}
	Let \(\dworld{}\) be a meta-distribution satisfying \(K\)-Sparsity
        (\cref{assum-k-sparse}) and Regular Facts
        (\cref{assum-regular-facts}). Let \(\mathcal{A}\) be a Language Model and
        \(g := \mathcal{A}(X)\) its predictive distribution over \(\Omega\) given input
        \(X \in \mathcal{M}(\Omega)\).  If \(g(U) > 0\) then
	\[
		\Pr_{p \sim \dworld{}}\bigl[\, g(H) > 0 \,\bigm| X \bigr]
		\;\ge\;
		1 - \frac{K}{|U|}.
	\]
\end{theorem}

\begin{proof}
	Since \(g(U)>0\), pick \(y^\star\in U\) with \(g(y^\star)>0\). For any \world{} \(p\), we have \(H = \Omega\setminus F\), so \(y^\star\in H\) if and only if \(y^\star\notin F\).
	By Regular Facts, all \(y\in U\) have the same probability \(q := \Pr[y\in F \mid X]\). Hence
	\begin{equation}
          \mathbb{E}\bigl[\,|F\cap U| \mid X \bigr]
          = \sum_{y\in U} \Pr[y\in F \mid X] 
          = |U|\, q.  \end{equation}
	By \(K\)-sparsity, \(|F|\le K\) for every \world{}, so \(|F\cap U|\le
        K\). Taking conditional expectations on both sides gives
        \(|U|\,q \le K\), so that \(q \le K/|U|\). Applying this to \(y^\star\),
	\[
		\Pr[y^\star\in H \mid X]
		= 1 - \Pr[y^\star\in F\mid X]
		\ge 1 - \tfrac{K}{|U|}.
	\]
	Whenever \(y^\star\in H\), we have \(g(H)\ge g(y^\star)>0\). Therefore
	\[
		\Pr[g(H)>0 \mid X]
		\ge \Pr[y^\star\in H \mid X]
		\ge 1 - \tfrac{K}{|U|}. \qedhere
	\] 
\end{proof}

\cref{prop:hall-to-inn,thm:innovation-implies-hallucination} together show that,
in the Kalai-Vempala formalism, innovation is an almost characterization of
hallucination. To compare this with the results of Kalai and Vempala, we now
turn to the relationship between innovation and their notion of calibration.

\subsection{Calibration, Missing Mass, and Innovation}

Kalai and Vempala, at the existence level, show that if there is significant
missing mass, any calibrated model must hallucinate.  We now show that
calibration together with any missing mass already implies innovation.  The
intuition is straightforward: calibration forces the model to assign positive
probability to every factual statement.  As soon as there is missing mass, at
least one such statement lies outside the observed corpus, and the model must
assign it positive probability. Hence the model necessarily innovates. We now
formalize this.

\begin{proposition}[Calibration and missing mass imply innovation]
	\label{prop:calib-innov}
	Let $p \in \Delta(\Omega)$ be the true \world{}, $X \sim p^{\times n}$ the training data,
        and $U := \Omega \setminus O$ the set of unobserved statements, where
        $O = \mathrm{set}(X)$.  Let \(\mathcal{A}\) be a language model and
        $g:=\mathcal{A}(X)$ be a predictive distribution that is calibrated to
        $p$, i.e., there exists a partition $\Pi$ of $\Omega$ with $g = p^\Pi$.  If there
        is any missing mass, i.e., $p(U)>0$, then $\mathcal{A}$ innovates on $X$.
\end{proposition}
Combining \cref{prop:calib-innov} with our earlier result that innovation forces
hallucination with high posterior probability under \(K\)-Sparsity and Regular
Facts (\cref{thm:innovation-implies-hallucination}), we obtain the refined
logical chain described in \cref{eq-refined-chain} in the introduction.
This shows that, at the level of existence, our result strictly generalizes the
Kalai-Vempala result.

A natural next question is how large the hallucination rate \(g(H)\) must be
once innovation occurs.  In the next section we show that the \emph{innovation
  rate} \(g(U)\) governs the \emph{rate} at which a model that innovates must
hallucinate: a fixed fraction of \(g(U)\) must fall on false statements with
high probability.

 \section{Lower Bounds on Hallucination Rate via Innovation Rate}
\label{sec:rates}
Having characterized the existence of hallucination, we now ask a quantitative
question: \emph{how much} must a model hallucinate once it innovates? In this
section we derive lower bounds on the hallucination rate \(g(H)\) in terms of
the \emph{innovation rate} \(g(U)\). Then, we relate innovation rate to missing mass to obtain bounds that complement those of Kalai and Vempala.

\subsection{A Markov-style Lower Bound on the Hallucination Rate}

We now derive a lower bound on \(g(H)\) in terms of the innovation
mass \(g(U)\) using a Markov-style argument.

\begin{theorem}[Markov-style bound]
	\label{thm:markov-bound}
	Let \(\dworld{}\) be a meta-distribution satisfying \(K\)-Sparsity and
        Regular Facts. Let \(\mathcal{A}\) be a Language Model and
        \(g := \mathcal{A}(X)\) its predictive distribution over \(\Omega\) given input
        \(X \in \mathcal{M}(\Omega)\). Then for any \(\delta \in (K/|U|, 1)\),
	\[
		\Pr_{p \sim \dworld{}}\insq*{\, g(H) \;\ge\; g(U)\,\left( 1 - \frac{K}{\delta |U|}\right) \,\Bigm|\, X}
		\;\ge\;
		1 - \delta.
	\]
\end{theorem}
\begin{proof}
  We begin the proof by lower bounding the expected hallucination rate. Write
  \(g(F\cap U)=\sum_{y\in U} g(y)\mathbf{1}[y\in F]\).  By the Regular Facts assumption,
  there exists a \(q\) such that \(q = \Pr[y\in F\mid X]\) for every
  \(y \in U\).  As in the proof of \cref{thm:innovation-implies-hallucination},
  \(K\)-Sparsity implies \(q\le K/|U|\).  Taking conditional expectations (and
  noting that \(U\) and \(g(y)\) are completely determined by \(X\), i.e., they
  are measurable with respect to the \(\sigma\)-field generated by the random
  variable \(X\)), we thus get
\(\mathbb{E}[g(F\cap U)\mid X] = \sum_{y\in U} g(y)\Pr[y\in F\mid X] = \Bigl(\sum_{y\in U} g(y)\Bigr) \cdot q = q g(U) \leq g(U)\,K/|U|. \) Since \(g(H)=g(U)-g(F\cap U)\),
  \begin{equation}
    \label{eq:1}
    \mathbb{E}[g(H)\mid X] \ge g(U)\bigl(1-K/|U|\bigr).
  \end{equation}
  Let \(t\in(0,1)\) be a parameter to be chosen later. We now invoke a
  Markov-style argument. First, we split the expected hallucination rate
  \( \mathbb{E}[g(H)\mid X]\) into two terms
  \( \mathbb{E}[g(H) \mathbf{1}[g(H) \ge t g(U)]\mid X]\) and
  \( \mathbb{E}[g(H) \mathbf{1}[g(H) < t g(U)]\mid X] \). Because
  \(g(H)\le g(U)\), the first term is upper bounded by
  \(g(U)\mathbb{E}[\mathbf{1}[g(H) \ge t g(U)]\mid X]\) and the second term is upper
  bounded by \( t g(U) \mathbb{E}[\mathbf{1}[g(H) < t g(U)]\mid X]\).  Hence, we
  have
  \begin{align*}
    \mathbb{E}[g(H)\mid X]
    \le & \alpha\,g(U) + (1-\alpha)\,t\,g(U) \\
    = & g(U)\,(t + (1-t)\alpha),
  \end{align*}
  where \(\alpha := \Pr[g(H)\ge t\,g(U)\mid X]\).  Combining with the lower bound on
  \(\mathbb{E}[g(H)\mid X]\) in \cref{eq:1} gives
\(\alpha
\;\ge\; \frac{1 - K/|U| - t}{1-t}.
\)  Substituting \(t = 1 - \displaystyle \frac{K}{\delta |U|}\) in the above and
  using \(K/|U| < \delta < 1\) now gives
  \[
    \Pr\insq*{g(H)\ge g(U)\cdot \inp*{1 - \displaystyle \frac{K}{\delta |U|}}\mid X} = \alpha \ge 1 -
    \delta. \qedhere
  \]
\end{proof}

\subsection{A High-Confidence bound on Hallucination Rate}

The Markov-style lower bound in Theorem~\ref{thm:markov-bound} becomes
vacuous as \(\delta\) approaches \(K/|U|\), since the multiplicative factor
\(1-K/(\delta |U|)\) then tends to zero.  Similarly, the Kalai-Vempala bounds in Corollary \ref{cor:kv-cor1} and  Corollary~\ref{cor:kv-cor2} become vacuous whenever
\(\delta < 2K/|U|\), because \(p(U) \le 1\) and \(\|g - p^\Pi\|_{\mathrm{TV}} \ge 0\)
imply that the right-hand side of those bounds is non-positive. To go beyond this
limitation, we now present a different argument that yields a nontrivial lower
bound on hallucination rate with probability at least \(1 - K/|U|\).

The intuition is simple. Any model that innovates must inevitably assign
probability to hallucination. If it spreads its probability broadly, sparsity
defeats it: almost all unseen statements are hallucinations. If it concentrates
sharply, regularity defeats it: having no basis to distinguish the unseen truths
from falsehoods, the spike almost surely lands in the wrong place. Informally,
every model must either scatter or spike and both strategies fail. We now
formalize this.
\begin{theorem}[High-Confidence bound]
	\label{thm:sharp-bound}
	Let \(\dworld{}\) be a meta-distribution satisfying \(K\)-Sparsity and Regular Facts. Let \(g = \mathcal{A}(X)\) be any language model. Then
	\[
		\Pr_{p \sim \dworld{}}\Bigl[\, g(H) \;\ge\; \frac{g(U)}{K+1} \,\Bigm|\, X \Bigr]
		\;\ge\;
		1 - \frac{K}{|U|}.
	\]
\end{theorem}

\begin{proof}
  Let \( m := \max_{y \in U} g(y) \) denote the largest mass assigned by \(g\) to
  any unseen statement, and fix \(y^{\star} \in U\) satisfying
  \(g(y^{\star}) = m\).  Note that \(U, m\), \(y^{\star}\), and \(g(U)\) are random
  variables, but they are completely determined given the training corpus \(X\)
  (formally, they are measurable with respect to the \(\sigma\)-field generated by
  the random variable \(X\)).  Note also that applying the same line of
  reasoning as in the proof of
  Theorem~\ref{thm:innovation-implies-hallucination} on \(y^\star\) gives
  \(\Pr[y^\star \in H \mid X] \geq 1 - \frac{K}{|U|}.\) We now analyze two cases.\\
  \noindent\textbf{Case 1: \(m \le g(U)/(K+1)\).}
  In this case, using \(|F \cap U| \le K\) and \(K\)-Sparsity, we get
  \( g(F \cap U) = \sum_{y \in F \cap U} g(y) \;\le\; K \cdot m \;\le\; K \cdot \frac{g(U)}{K+1},\) so
  that
  \( g(H) = g(U) - g(F \cap U) \;\ge\; g(U) - K\cdot \frac{g(U)}{K+1} = \frac{g(U)}{K+1}.\\
  \) \noindent\textbf{Case 2: \(m > g(U)/(K+1)\).}  In this case
\( g(y^\star) > \frac{g(U)}{K+1},\) so that the event \(y^{*} \in H\) implies the
  event \(g(H) \geq g(U)/(K+1)\).

  Let \(\mathcal{E}\) denote the event \(g(H) \geq g(U)/(K+1)\), and let
  \(\mathcal{E}_1\) and \(\mathcal{E}_2\) denote the events corresponding to the two cases: note
  that \(\mathcal{E}_1\) and \(\mathcal{E}_2\) are determined by \(X\) (since
  \(m, U\) and \(g(U)\) are).  Writing the above conclusions in terms of
  indicator random variables thus gives
  \(\Pr[\mathcal{E}|X] \geq I_{\mathcal{E}_1} + I_{\mathcal{E}_2}\Pr[y^{\star} \in H|X]\).  Since
  \(\Pr[y^\star \in H \mid X] \geq 1 - \frac{K}{|U|}\), the claim follows by noting that
  \(I_{\mathcal{E}_{1}} + I_{\mathcal{E}_2} = 1\).
\end{proof}

We thus see that the only way to avoid hallucination entirely in this framework
is to choose a model that never innovates, i.e., a model with \(g(U) = 0\). Any
attempt to generalize beyond the observed support inevitably incurs
hallucination at a rate proportional to the model's innovation rate.

\mparagraph{The Regular Probabilities Assumption} We emphasize that both
Theorem~\ref{thm:markov-bound} and Theorem~\ref{thm:sharp-bound} rely only on
\(K\)-Sparsity and Regular Facts, and do \emph{not} assume Regular Probabilities
(\cref{assum-regular-prob}). In contrast, the Kalai--Vempala bound for
meta-distributions satisfying \(K\)-Sparsity and Regular Facts alone
(Corollary~\ref{cor:kv-cor2}) incurs an error term of order \(Kn/|U|\).  As the
corpus size grows (whenever \(n \ge |U|/K\)), this bound becomes vacuous, which
might suggest that with sufficiently large training data one could drive the
hallucination rate to zero under these assumptions.  Corollary 1 of
\citet{kalaivempala2024caliberatedllms} (\cref{cor:kv-cor1} above) shows that
this intuition is false, at least under an additional Regular Probabilities
assumption.  Our results show that this intuition is false even without this
assumption: even if the corpus size \(n\) is large, a model with a positive
innovation rate \(g(U)\) must hallucinate at a comparable rate.

\subsection{From Innovation Rate to Missing Mass}

We now show how to translate bounds expressed in terms of \emph{innovation rate}
\(g(U)\), which is a property of the model, into bounds expressed in terms of
the \world{}'s \emph{missing mass} \(p(U)\), which is a property of the \world{}
and the corpus \(X\).  This crucial translation also allows us to extend our
results in a way that complements the bounds of Kalai and Vempala.
\begin{proposition}[Missing mass lower-bounds innovation rate]
	\label{thm:gU-missingmass}
	Assume \(K\)-Sparsity. Let \(p \sim \dworld{}\),  \(X \sim p^{{\times n}}\), and let \(O := \set(X)\), \(U := \Omega \setminus O\).
	Let \(\Pi\) be a partition of \(\Omega\) and let \(p^\Pi\) denote the \(\Pi\)-coarsening
	of \(p\) as defined in Definition~\ref{def:coarsening}. Let \(\mathcal{A}\) be a language model and \(g := \mathcal{A}(X)\) be its predictive distribution. Then, we have
	\[
		g(U) \;\ge\; \frac{p(U)}{K+1} \;-\; \|g - p^\Pi\|_{\mathrm{TV}}.
	\]
\end{proposition}
\Cref{thm:gU-missingmass} is the link we need: note also that the proposition
applies to any partition \(\Pi\), even if \(\Pi\) is chosen depending upon
\(X\), and the (unknown) true document distribution \(p\).  Combining it with
\cref{thm:markov-bound,thm:sharp-bound}, we get the following corollaries:

\begin{corollary}[A missing mass bound from the Markov-style bound]
	\label{cor:markov-missing}
	Under the assumptions of \cref{thm:markov-bound}, for any
        \(\delta\in(K/|U|,1)\), with probability at least \(1-\delta\) conditioned on the
        corpus \(X\),
	\begin{align*}
		g(H) & \ge \frac{p(U)}{K+1} - \frac{1}{\delta|U|} - \|g - p^\Pi\|_{\mathrm{TV}},
	\end{align*}
        where \(\Pi\) is any partition of \(\Omega\), possibly dependent upon both
        \(X\) and the (unknown) true document distribution \(p\).  For the
        special case
        \(\Pi=\mathcal B_g\) (level sets of \(g\)), this guarantee becomes
	\[
		\Pr\!\left[
			g(H) \ge\frac{p(U)}{K+1} - \frac{1}{\delta|U|} - \mathrm{Mis}(g,p)
			\,\Bigm|\, X
			\right]
		\;\ge\;
		1 - \delta.
	\]
\end{corollary}
\begin{corollary}[A missing mass bound from the high-confidence bound]
	\label{cor:sharp-missing}
	Under the assumptions of Theorem~\ref{thm:sharp-bound}, with probability at
	least \(1 - \dfrac{K}{|U|}\) conditioned on the corpus \(X\),
	\begin{align*}
		g(H)
		 & \;\ge\;
\frac{p(U)}{(K+1)^2} - \frac{\|g - p^\Pi\|_{\mathrm{TV}}}{K+1},
	\end{align*}
        where \(\Pi\) is any partition of \(\Omega\), possibly dependent upon both
        \(X\) and the (unknown) true document distribution \(p\).  For the
        special case
        \(\Pi=\mathcal B_g\) (level sets of \(g\)), this guarantee becomes
	\[
		\Pr\!\left[
			g(H) \ge \frac{p(U)}{(K+1)^2} - \frac{\mathrm{Mis}(g,p)}{K+1}
			\,\Bigm|\, X
			\right]
		\;\ge\;
		1 - \frac{K}{|U|}.
	\]
\end{corollary}

\mparagraph{Comparison with the Kalai-Vempala Bounds} These corollaries
complement the Kalai-Vempala bounds in a precise sense.  First,
\cref{cor:sharp-missing} yields a missing mass lower bound on the hallucination
rate in a regime for the error probability \(\delta\) where both the Kalai-Vempala
bounds and the Markov-style bounds collapse: In particular, the Kalai-Vempala
bounds (in both
\cref{cor:kv-cor1,cor:kv-cor2}) become vacuous for \(\delta < 2K/|U|\), while
\cref{cor:sharp-missing} continues to provide a nontrivial guarantee with error
probability at most \(\delta = K/|U|\).

Second, \Cref{cor:markov-missing} gives a missing mass lower bound in the high
error probability regime under \(K\)-Sparsity and Regular Facts alone.  Compared
to the corresponding bound obtained by Kalai and Vempala (\cref{cor:kv-cor2}),
this bound is weaker by a factor of \(1/(K+1)\) in its dependence on the missing
mass.  However, compared to the Kalai-Vempala bounds, it has the advantage of
having no dependence on the corpus size \(n\). Consequently, it ensures a
hallucination rate on the order of \(p(U)/(K+1)\) even when \(n\) is comparable
to \(|U|/K\), a regime in which the Kalai-Vempala bound under \(K\)-Sparsity and
Regular Facts (Corollary \ref{cor:kv-cor2}) alone is vacuous.
Thus, compared to the Kalai-Vempala bound that additionally assumes Regular
Probabilities (Corollary \ref{cor:kv-cor1}) in order to remove the dependence on
corpus size, our \cref{cor:markov-missing} can be viewed as a qualitative
recovery of their conclusion under strictly weaker assumptions. While the
quantitative dependence on the missing mass is looser by a factor of
\(1/(K+1)\), the interpretation remains the same: any calibrated language model
must hallucinate at a rate controlled by the missing mass.

 \section{Empirical explorations} \label{sec:experiments}

The Kalai-Vempala framework and our results are aimed at model trainers with
access to training data. However, rarely is the training data available for
commercial LLMs. For an empirical exploration of our results, we therefore
consider a controlled toy setting where we have full access to both the model
and the training data. Towards this, we begin by replicating the experimental
setup of \citet{miaoHallucinationMonofactsMiscalibration2025} based on
\(n\)-gram models, that was designed to study the results of
\citet{kalaivempala2024caliberatedllms}. In this simple experiment, we find, as
expected, that whenever a model innovates, it also hallucinates and
therefore, innovation rate exactly matches hallucination rate. We defer the
details of this baseline experiment to \cref{sec:tuples-experiment}. We then
move to a more realistic setting, which adds to the \(n\)-gram model of
\citet{miaoHallucinationMonofactsMiscalibration2025} a rudimentary notion of
semantic truth. A core assumption of the Kalai-Vempala framework we use is thus
violated, and the experiment is thus also a test of how robust our conclusions
are to such perturbations.
 
\mparagraph{Data} We use a publicly available
dataset\footnote{Available from the UC Irvine Machine Learning Repository at
  \url{https://archive.ics.uci.edu/dataset/331/sentiment+labelled+sentences}
  under a CC-BY 4.0 license.} of approximately 3000 customer reviews
\citep{Kotzias2015FromGT}, retaining only the review text, removing all
punctuation and capitalization, and restricting to reviews of at most 20 words.
After preprocessing, the dataset contains 2350 unique reviews, with a vocabulary
of 3722 unique words.

\mparagraph{Methodology} We train $n$-gram models for $n \in \{2, 3, \ldots, 7\}$ using the \texttt{nltk} library and generate multiple outputs for each model.
The \emph{innovation rate} is the fraction of generated statements not found in
the training data. We evaluate the hallucination rate using five judges:
\begin{enumerate}
    \item \textbf{Human:} We manually label approximately 900 generated
          statements by asking ourselves if the generated statement  could be a review (excluding duplicates and statements already in the training data).
    \item \textbf{LLM judges:} We query four foundation models via \texttt{OpenRouter} using the
          prompt: ``{Is the following text a review? Respond with a 1 if it is, or with a 0
              if it isn't. $\langle$BEGIN TEXT$\rangle$ \{text of the generated
              statement\} $\langle$END TEXT$\rangle$}''
\end{enumerate}

\begin{figure}[t!]
    \centering
    \includegraphics[width=0.48\textwidth]{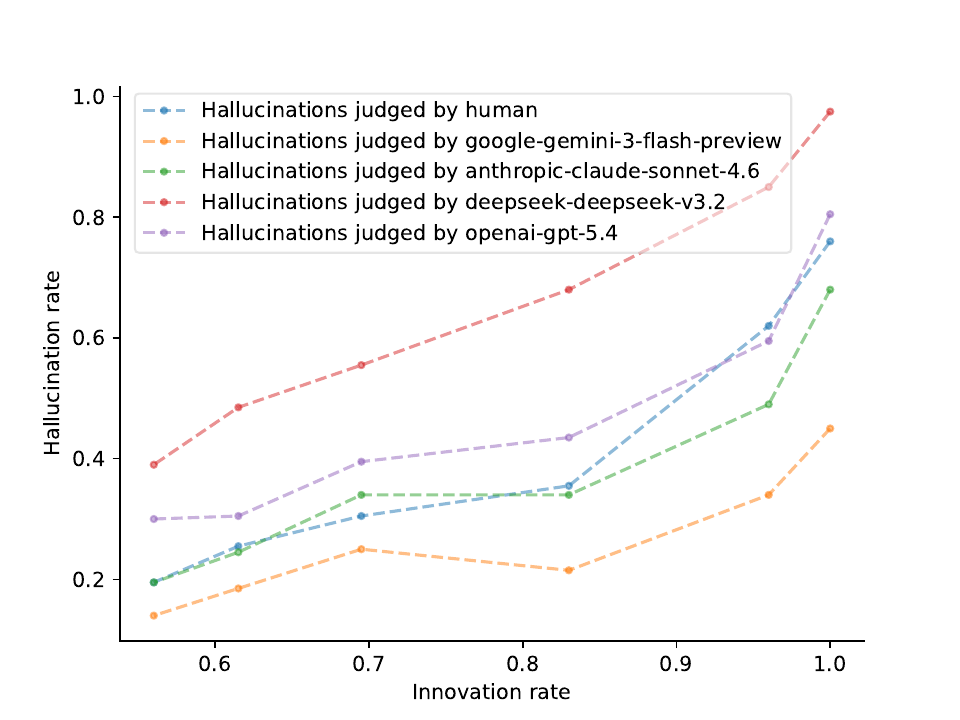}
    \caption{Hallucination Rate versus Innovation Rate. Each point denotes the innovation rate and a hallucination rate assessed by a judge (denoted by the colour) for a particular $n$-gram model. 
}
    \label{fig:innov-vs-hallu}
\end{figure}

\mparagraph{Results} Figure~\ref{fig:innov-vs-hallu} plots hallucination rate against
innovation rate for each judge, with each point corresponding to a different
$n$-gram order. We find that for each judge, innovation rate closely tracks hallucination rate across all the $n$-gram models. This suggests that innovation rate, which is easy to estimate for a model trainer, is a plausible indicator of the hallucination rate of language models. Further, given that the assumptions of our theoretical results are likely violated in this setting, the empirical results suggest that the relationship between innovation and hallucination is robust to deviations from these assumptions.  \cref{appendix: experiments} contains additional data and plots for this experiment.

 \section{Discussion}
\label{sec:discussion}

\mparagraph{Estimating the Hallucination Rate of LLMs} A natural application of quantities such as calibration and innovation, together with their connection to hallucination established by \citet{kalaivempala2024caliberatedllms} and extended in this work, is the estimation of hallucination rates in LLMs.\footnote{The Kalai--Vempala framework is formulated in terms of missing mass and miscalibration, both of which require indirect (proxy-based) estimation. In particular, missing mass can be estimated via the Good--Turing estimator \citep{goodPopulationFrequenciesSpecies1953}, as noted by \citet{kalaivempala2024caliberatedllms}, while miscalibration can be estimated using methods proposed by \citet{miaoHallucinationMonofactsMiscalibration2025}. In contrast, the innovation rate admits a direct empirical estimator.} Our empirical results indicate that innovation rate closely tracks hallucination rate. However, it remains unclear whether this relationship persists more broadly, particularly for modern LLMs.
The robustness of the relationship between innovation and hallucination to
deviations from the assumptions of our theoretical results, as observed in our
empirical results, thus invites further empirical and theoretical investigations.

\mparagraph{Feedback and Hallucination}  One way to interpret our results is that
hallucination can be seen as a consequence of innovation in the absence of
semantic feedback.  In practical settings, where model outputs can be evaluated
and models can be retrained or adjusted based on this feedback, such feedback
mechanisms may allow hallucination to be mitigated even while innovation
persists.  

We include further miscellaneous comments in \cref{sec-misc-comments}.

\section{Conclusion}

The outstanding question in the theoretical modeling of hallucination is to design theoretical frameworks that can capture as many aspects of the phenomenon
as possible.  One way to test existing frameworks, such as the framework of
\citet{kalaivempala2024caliberatedllms} studied in the paper, is to probe the
framework by asking what are the qualitatively strongest conclusions that are
implied by it.  Equivalently, the question is to identify the qualitatively
\emph{weakest} conditions under which the framework implies hallucination.  Our
work contributes to this line of work by identifying our notion of innovation as
the weakest possible condition that implies hallucination in the Kalai-Vempala
framework.

\section*{Acknowledgements}
We thank the anonymous reviewers for several helpful comments and suggestions.
We gratefully acknowledge support from the Department of Atomic Energy,
Government of India [project numbers RTI4001 and RTI4014]; by the
Infosys-Chandrasekharan virtual center for Random Geometry at the Tata Institute
of Fundamental Research; by the Science and Engineering Research Board [grant
number MATRICS MTR/2023/001547]; and by gifts to TIFR from Adobe Systems
Incorporated and through a Google India Research Award. The contents of this
paper do not necessarily reflect the views of the funding agencies listed above.
Large language models were used for proofreading, for feedback on readability,
and for help in programming the empirical explorations. We especially
acknowledge the pre-submission ``LLM feedback'' provided by ICML using the
Google Paper Assistant Tool.

\section*{Impact Statement}
The goal of this paper is to advance our understanding of high-level stochastic
modeling of the properties of language models.  Beyond this goal, we do not
believe there are any potential societal consequences of this work that need to
be highlighted here.

\newpage
\appendix
 \onecolumn
\section{Relaxing the Regular Facts Assumption}
\label{app:r-regular-facts}

A central assumption in the Kalai--Vempala framework is \emph{Regular Facts}
(\cref{assum-regular-facts}), under which, conditioned on the observed
corpus, all unobserved statements are equally likely to be factual.
Kalai and Vempala~\citeyearpar{kalaivempala2024caliberatedllms} introduce a natural
relaxation of this assumption, allowing limited non-uniformity.
This relaxation is technically straightforward and we describe its
 effect on our bounds in this appendix. We show that the only effect of relaxing Regular Facts to \(r\)-Regular Facts is a degradation of constants and failure probabilities proportionate to the relaxation.

\begin{assumption}[\detailcite{Definition 3: \(r\)-Regular Facts}{kalaivempala2024caliberatedllms}]
	\label{assum:r-regular-facts}
	We say that a meta-distribution \(\dworld{}\) satisfies \emph{\(r\)-Regular
	Facts} if for every corpus \(X\) and all unobserved statements
	\(y \in U := \Omega \setminus \set(X)\),
	\[
		\Pr_{p \sim \dworld{}}[y \in F \mid X]
		\;\le\;
		r \cdot \frac{\mathbb{E}_{p \sim \dworld{}}[|F \cap U| \mid X]}{|U|}.
	\]
\end{assumption}

The case \(r = 1\) recovers Regular Facts
(\cref{assum-regular-facts}). Intuitively,
\(r\)-Regular Facts allows posterior factual probabilities of unseen statements to vary by at most a multiplicative factor of \(r\) from their mean.

\subsection{Markov-Style Bound under \(r\)-Regular Facts}

We now state the extension of the Markov-style hallucination bound.

\begin{theorem}[Markov-style bound under \(r\)-Regular Facts]
	\label{thm:markov-r-regular}
	Assume \(K\)-Sparsity and \(r\)-Regular Facts (Assumptions \ref{assum-k-sparse} and \ref{assum:r-regular-facts}). Let \(g = \mathcal{A}(X)\) be the
	predictive distribution of a language model. Then for any
	\(\delta \in (rK/|U|, 1)\),
	\[
		\Pr_{p \sim \dworld{}}\Bigl[
			g(H) \;\ge\;
			g(U)\Bigl(1 - \frac{rK}{\delta |U|}\Bigr)
			\;\Bigm|\; X
		\Bigr]
		\;\ge\;
		1 - \delta.
	\]
\end{theorem}

\begin{proof}
	As in the proof of Theorem~\ref{thm:markov-bound}, we begin by bounding the
	expected factual mass assigned to unseen statements. By linearity of
	expectation,
	\[
		\mathbb{E}[g(F \cap U) \mid X]
		=
		\sum_{y \in U} g(y)\Pr[y \in F \mid X].
	\]
	By \(r\)-Regular Facts and \(K\)-Sparsity,
	\[
		\Pr[y \in F \mid X]
		\;\le\;
		r \cdot \frac{\mathbb{E}[|F \cap U| \mid X]}{|U|} \;\le\; r \cdot \frac{K}{|U|}.
	\]
	Substituting yields 
	\[
		\mathbb{E}[g(F \cap U) \mid X]
		\;\le\;
		r \cdot \frac{K}{|U|} \sum_{y \in U} g(y)
		=
		r \cdot g(U)\frac{K}{|U|}.
	\]
	Since \(g(H) = g(U) - g(F \cap U)\), it follows that
	\[
		\mathbb{E}[g(H) \mid X]
		\;\ge\;
		g(U)\Bigl(1 - \frac{rK}{|U|}\Bigr).
	\]
	The rest of the proof proceeds exactly as in
        Theorem~\ref{thm:markov-bound} (with the parameter \(t\) in that proof
        now chosen to be \(1 - \frac{rK}{\delta\abs{U}}\)), yielding the stated
        probability bound.
\end{proof}

\subsection{High-Confidence Bound under \(r\)-Regular Facts}

We next describe the effect of \(r\)-Regular Facts on the high-confidence bound.

\begin{theorem}[High-confidence bound under \(r\)-Regular Facts]
	\label{thm:sharp-r-regular}
	Assume \(K\)-Sparsity and \(r\)-Regular Facts (Assumptions \ref{assum-k-sparse} and \ref{assum:r-regular-facts}). Let \(g = \mathcal{A}(X)\) be the predictive distribution of a language model. Then
	\[
		\Pr_{p \sim \dworld{}}\Bigl[
			g(H) \;\ge\; \frac{g(U)}{K+1}
			\;\Bigm|\; X
		\Bigr]
		\;\ge\;
		1 - \frac{rK}{|U|}.
	\]
\end{theorem}

\begin{proof}
  The proof follows that of Theorem~\ref{thm:sharp-bound}, except that the bound
  \( \Pr[y^\star \in F \mid X] \le K/|U| \) is replaced under \(r\)-Regular Facts by
  \( \Pr[y^\star \in F \mid X] \le rK/|U|\).  The remainder of the argument is unchanged.
\end{proof}

\subsection{Consequences for Missing Mass Bounds}

Since the relationship between innovation rate and missing mass
(\cref{thm:gU-missingmass}) does not rely on Regular Facts, all
missing mass corollaries in \cref{sec:rates} extend immediately to
\(r\)-Regular Facts by combining Theorems~\ref{thm:markov-r-regular} and
\ref{thm:sharp-r-regular} with the same translation arguments. The effect of
\(r\)-Regular Facts is limited to a multiplicative degradation of constants or,
in the high-confidence case, a degradation of the confidence level.

 \section{Proofs Omitted from the Main Paper}
\begin{proof}[Proof of \cref{prop:hall-to-inn}]
  By definition, \(H = \Omega \setminus F\) and \(F \subseteq \Omega\). Since
  \(O \subseteq F\) (as every statement in the training data is factual), every hallucination is
  unobserved, i.e. \(H \subseteq U\), so that \( g(U) \ge g(H)\). Thus,
  \(g(H) > 0\) implies \(g(U) > 0\).
\end{proof}

\begin{proof}[Proof of \cref{prop:calib-innov}]
	By calibration, there exists a partition $\Pi$ of $\Omega$ such that $g = p^\Pi$.
Now suppose that $p(U)>0$. Then there exists some $y^\star \in U$ with
        $p(y^\star)>0$.  Let $B^\star$ be the cell of $\Pi$ containing
        $y^\star$. Since $p(y^\star)>0$, we have $p(B^\star) \ge p(y^\star)>0$, and therefore
        \( g(y^\star) = \frac{p(B^\star)}{|B^\star|} > 0.  \) Because
        $y^\star \in U$, this implies \( g(U) \;\ge\; g(y^\star) \;>\; 0.\)
\end{proof}

\subsection{Proof of \cref{thm:gU-missingmass}}

To prove \cref{thm:gU-missingmass}, we begin with
a simple but key observation that lower-bounds the amount of probability mass
that any coarsening can remove from the set \(U\).

\begin{lemma}[Coarsening preserves a $1/(K+1)$-fraction of the missing mass]
	\label{lem:coarsening-missing-mass}
	Let \(p \in \Delta(\Omega)\) such that \(|\supp(p)| \le K\). Let  \(X \sim p^{\times n}\), and \(O := \set(X)\), \(U := \Omega \setminus O\).
	Let \(\Pi\) be a partition of \(\Omega\) and let \(p^\Pi\) denote the \(\Pi\)-coarsening
	of \(p\) as defined in Definition~\ref{def:coarsening}. Then, we have
	\[
		p^\Pi(U) \;\ge\; \frac{p(U)}{K+1}.
	\]
\end{lemma}

\begin{proof}[Proof of \cref{lem:coarsening-missing-mass}]
  Fix a partition $\Pi$ of $\Omega$.  First, let us show that the inequality holds for
  all the cells in the partition \(\Pi\). Fix any \(B \in \Pi\). We want to show the
  following
	\[
		p^\Pi(U \cap B) \;\ge\; \frac{p(U \cap B)}{K+1},
		\label{eq:cellwise}
	\]
	If \(U \cap B = \emptyset\), then \(p^\Pi(U \cap B) = p(U \cap B) = 0\) and the above
        inequality holds trivially. Hence, we only have to argue for the case
        \(U \cap B \ne \emptyset\). We argue as follows.
	\begin{align*}
		p^\Pi(U \cap B) & = |U \cap B| \cdot \frac{p(B)}{|U \cap B|+|O \cap B|} \quad          & \tag{1} \\
		                & \ge p(U \cap B) \cdot \frac{|U \cap B|}{|U \cap B|+|O \cap B|} \quad & \tag{2} \\
		                & \ge p(U \cap B) \cdot \frac{|U \cap B|}{|U \cap B|+ K} \quad         & \tag{3} \\
		                & = p(U \cap B) \cdot \frac{1}{1+ \frac{K}{|U \cap B|}}                          \\
		                & \ge p(U \cap B) \cdot \frac{1}{1+K}                                  & \tag{4}
	\end{align*}
	where, (1) follows from the coarsening definition, (2) follows from \(U \cap B \subseteq B \), (3) uses \(|O \cap B| \le |O| \le |F| \le K\), and (4) follows from \(U \cap B \ne \emptyset \implies |U \cap B| \ge 1 \implies \frac{K}{|U \cap B|} \le K\).

	Hence, for all the cells of the partition, the inequality holds. Summing over all the cells finishes the proof.
\end{proof}
Using this lemma, we prove \cref{thm:gU-missingmass} which relates innovation rate and missing mass. 

\begin{proof}[Proof of \cref{thm:gU-missingmass}]
	Total variation satisfies
	\[
		\|g - p^\Pi\|_{\mathrm{TV}}
		\ge p^\Pi(U) - g(U).
	\]
	Rearranging gives \(g(U) \ge p^\Pi(U) - \|g - p^\Pi\|_{\mathrm{TV}}\).
	Applying \cref{lem:coarsening-missing-mass} completes the proof.
\end{proof}

\subsection{Proofs of \cref{cor:markov-missing,cor:sharp-missing}}

\begin{proof}[Proof of \cref{cor:markov-missing}]
  \Cref{thm:markov-bound} implies that with probability at least \(1-\delta\)
  conditioned on the corpus \(X\), it holds that
  \begin{equation}
    \label{eq:2}
    g(H) \geq g(U)\inp*{1 - \frac{K}{\delta\abs{U}}}.
  \end{equation}
  Substituting the bound
  \(g(U) \geq \frac{p(U)}{K + 1} - \|g - p^\Pi\|_{\mathrm{TV}}\) from
  \cref{thm:gU-missingmass} in \cref{eq:2} gives
  \begin{align}
    g(H) & \ge \Bigl(\frac{p(U)}{K+1} - \|g - p^\Pi\|_{\mathrm{TV}}\Bigr) \Bigl(1 - \frac{K}{\delta |U|}\Bigr) \\
         & \ge \frac{p(U)}{K+1} - \frac{K}{(K+1)\delta|U|} - \|g - p^\Pi\|_{\mathrm{TV}}           \label{eq:3}            \\
         & \ge
           \frac{p(U)}{K+1} - \frac{1}{\delta|U|} - \|g - p^\Pi\|_{\mathrm{TV}},
  \end{align}
  where \cref{eq:3} uses the conditions \(p(U) \leq 1\) and \(\delta > K/\abs{U}\).
\end{proof}

\begin{proof}[Proof of \cref{cor:sharp-missing}]
  \Cref{thm:sharp-bound} implies that with probability at least \(1-K/\abs{U}\)
  conditioned on the corpus \(X\), it holds that
  \begin{equation}
    \label{eq:4}
    g(H) \geq \frac{g(U)}{K+1}.
  \end{equation}
  Substituting the bound
  \(g(U) \geq \frac{p(U)}{K + 1} - \|g - p^\Pi\|_{\mathrm{TV}}\) from
  \cref{thm:gU-missingmass} in \cref{eq:4} gives the claim.
\end{proof}

 \section{Miscellaneous Discussion}
\label{sec-misc-comments}

\mparagraph{Hallucination in Augmented Models} One way to model augmented generation (e.g., RAG) in the framework of this work is to enlarge the effective observed set available to the model at inference time: one can treat the model together with the retrieval mechanism as a single combined system, and the union of the training data and the data available for retrieval as the effective observed corpus. This view provides one possible explanation for why such methods may fail to eliminate hallucination in practice: even with access to a large external corpus, the combined system may continue to innovate, i.e. assign probability mass beyond the effective observed corpus, and hence continue to hallucinate.

\mparagraph{Extension to Large Training Corpus} A potential method towards extending the notion of innovation to settings where the training corpus is large is to define a distance structure on the space \(\Omega\) of statements, and to define innovation in terms of the model producing statements that are ``far'' from the statements observed during training. Such a distance structure would additionally induce a notion of semantic similarity between statements, allowing the notion of innovation to capture semantic novelty. In a practical setting, such a distance structure could arise from viewing \(\Omega\) as a set of vector embeddings in Euclidean space. We consider a concrete version of this idea in our empirical explorations through the semantic innovation rate, described in \cref{appendix: experiments}.
 \section{Additional Information for \cref{sec:experiments}} \label{appendix: experiments}

In addition to the innovation rate, we also consider a \emph{semantic innovation rate} which counts a statement as an innovation only if its cosine similarity to every training statement, as
computed by the \texttt{all-MiniLM-L6-v2} sentence transformer, falls below
$0.95$. The motivation for considering such a semantic innovation rate is twofold. First, it is scalable to settings where the training data is very large and hence, it is not feasible to check if a generated statement is present in the training data. Second, it is possible that a generated statement is not present verbatim in the training data but is semantically very close to a statement in the training data. Considering such generated statements as innovation is against the spirit of the innovation definition. We find that the semantic innovation rate also closely tracks hallucination rate across all the judges and $n$-gram models, as shown in \cref{fig:plot-errors}. Further, \cref{fig:plot-errors} also plots error bars obtained using the 95\%-confidence Clopper-Pearson interval.

\begin{figure}
    \centering
    \includegraphics[width=0.6\textwidth]{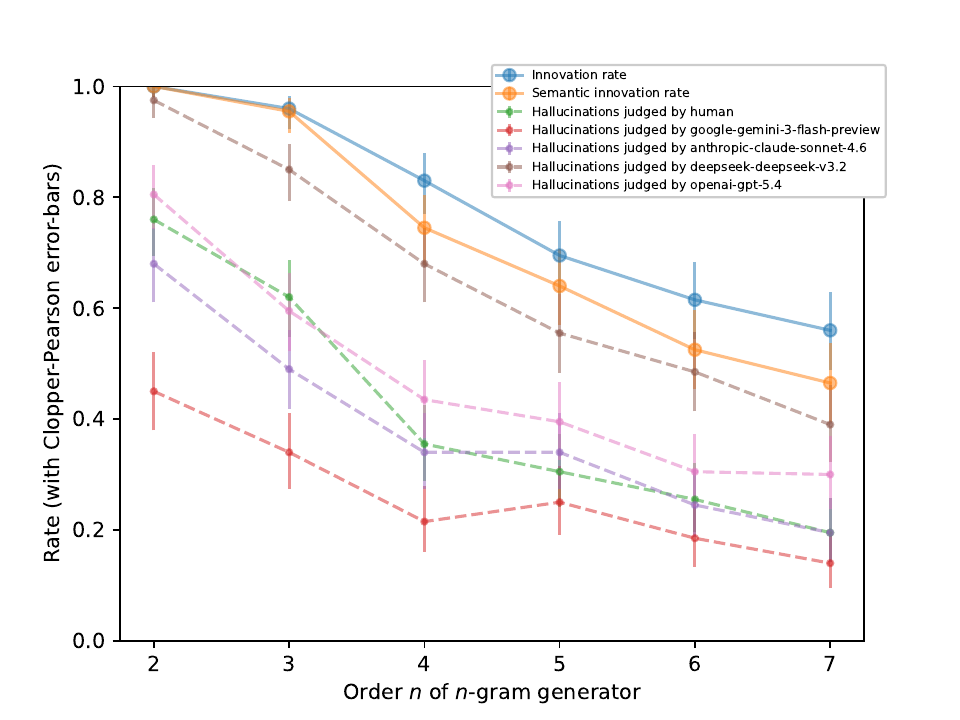}
    \caption{The two solid lines correspond to the observed innovation and
        semantic innovation rates, while the dashed lines plot the hallucination
        rates as judged by a human or different foundation models.  The error bars
        are obtained using the 95\%-confidence Clopper-Pearson interval. The
        points on each curve correspond to different $n$-gram models, for
        \(n \in \{2,3, \dots, 7\}.\)}
    \label{fig:plot-errors}
\end{figure}

\subsection{Replicating the Experiment of \citet{miaoHallucinationMonofactsMiscalibration2025}}
\label{sec:tuples-experiment}
\citet{miaoHallucinationMonofactsMiscalibration2025} designed an experimental
setup to study the Kalai-Vempala results which we replicate to validate our
results. First, a dataset of 10000 7-tuples of the form \emph{(Name, Date of
  Birth, Birthplace, Degree, College/University, Job, Employer)} is created and
then, a training corpus is created by sampling with replacement from this
dataset.\footnote{We consider the dataset used by Miao and Kearns which is
  available at \url{https://github.com/mmiao2/Hallucination}.} Each tuple in the
dataset is considered a ``fact''. We train $n$-gram models for
$n \in \{2, 3, 4, 5\}$ on the training corpus using the \texttt{ntlk} library and
generate multiple outputs for each \(n\)-gram model. A generated 7-tuple is
labeled an innovation if it is absent in the training corpus and a hallucination
if it does not appear in the full dataset (from which the training corpus was
sampled).

Given the lack of structure in the data, it seems plausible that whenever a model trained on such a training corpus innovates, i.e., the model produces a tuple outside the training corpus, it would likely hallucinate. Our experiments confirm that this is indeed the case and hence, in this simple setting, the innovation rate and the hallucination rate are exactly equal. We also observe that \(n\)-gram models trained with \(n \geq 4\) rarely innovate. Thus, while the empirical observations for this setting support our results, this might have been expected beforehand.

\subsection{Code Availability}

The full codebase, together with the datasets used in our experiments, is
available at
\url{https://github.com/nishantpratimdas/innovation-empirical-explorations}.
Some of the models used in our experiments may eventually be deprecated, making
it impossible to reproduce their outputs. For such cases, we have included all
the generated outputs in the repository to ensure that the experimental results
remain as reproducible as possible. Furthermore, the codebase supports
reproducing the full experimental pipeline on alternative and future models with
only minimal changes.

\end{document}